\documentclass[runningheads,a4paper]{llncs}

\usepackage[polish,english]{babel}
\usepackage[T1,nomathsymbols]{polski}
\usepackage[utf8]{inputenc}
\usepackage[T1]{fontenc}

\usepackage{caption}
\usepackage{subcaption}
\usepackage{amssymb}
\usepackage{graphicx}

\usepackage{xspace}

\usepackage{graphicx}
\usepackage{amsmath}

\usepackage{wrapfig,lipsum,booktabs}

\usepackage{tikz}
\usetikzlibrary{arrows,backgrounds,automata,topaths,patterns,decorations.pathreplacing,decorations.pathmorphing}
\usetikzlibrary{calc}
\usetikzlibrary{fit}
\usetikzlibrary{arrows}

\usepackage{times}

\usepackage{xcolor}

\newcommand{\OPRAm}[0]{\ensuremath{\mathcal{OPRA}_m}}
\newcommand{\QS}{$\mathcal{QS}$\xspace}
\newcommand{\Pred}[1]{ {\footnotesize \texttt{#1}} }

\newcommand{\Const}[1]{\operatorname{#1}}

\newcommand{\mysubsecNoBook}[2]{\noindent\textbf{#1}\quad}

\makeatletter
\renewcommand*{\thetable}{\arabic{table}}
\renewcommand*{\thefigure}{\arabic{figure}}
\let\c@table\c@figure
\makeatother 

\newcommand{\bul}{$\triangleright~$}

\definecolor{DarkBlue}{RGB}{7,72,110}
\definecolor{LightBlue}{RGB}{205,237,249}
\definecolor{LightRed}{RGB}{255,102,102}
\definecolor{DarkRed}{RGB}{153,0,0}
\definecolor{LightGreen}{RGB}{178,205,102}
\definecolor{DarkGreen}{RGB}{0,51,0}
\usepackage{url}
\urldef{\mailsa}\path|{alfred.hofmann, ursula.barth, ingrid.haas, frank.holzwarth,|
\urldef{\mailsb}\path|anna.kramer, leonie.kunz, christine.reiss, nicole.sator,|
\urldef{\mailsc}\path|erika.siebert-cole, peter.strasser, lncs}@springer.com|    
\newcommand{\keywords}[1]{\par\addvspace\baselineskip
\noindent\keywordname\enspace\ignorespaces#1}

\usepackage{hyperref}

\hypersetup{
    bookmarks=true,         
    pdftoolbar=true,        
    pdfmenubar=true,        
    pdffitwindow=false,     
    pdfstartview={FitH},    
    pdftitle={ASPMT(QS): Non-Monotonic Spatial Reasoning},    
    pdfauthor={Todo},     
    pdfnewwindow=true,      
    colorlinks=true,       
    linkcolor=blue,        
    linkbordercolor=white,
    citecolor=blue!50!black,        
    urlcolor=blue!80!black           
}

\begin{document}

\mainmatter  

\title{ASPMT(QS): Non-Monotonic Spatial Reasoning With \\ Answer Set Programming Modulo Theories}

\titlerunning{ASPMT(QS): Non-Monotonic Spatial Reasoning}

%
%
\author{ Przemys\l{}aw Andrzej Wa{\l}\k{e}ga\inst{1} \and Mehul Bhatt\inst{2} \and Carl Schultz\inst{2} 
}
\authorrunning{P. Wa{\l}\k{e}ga, M. Bhatt, C. Schultz}

\institute{ University of Warsaw, Institute of Philosophy, Poland,
\and
University of Bremen, Department of Computer Science, Germany.}

%
%

\toctitle{ASPMT(QS): Non-monotonic spatial reasoning}
\tocauthor{Authors' Instructions}
\maketitle

\begin{abstract}
The systematic modelling of \emph{dynamic spatial systems} \cite{bhatt:scc:08} is a key requirement in a wide range of application areas such as comonsense cognitive robotics, computer-aided architecture design, dynamic geographic information systems. We present ASPMT(QS), a novel approach and fully-implemented prototype for non-monotonic spatial reasoning ---a crucial requirement within dynamic spatial systems-- based on Answer Set Programming Modulo Theories (ASPMT). ASPMT(QS) consists of a (qualitative) spatial representation module (QS) and a method for turning tight ASPMT instances into Sat Modulo Theories (SMT) instances in order to compute stable models by means of SMT solvers. We formalise and implement concepts of  default spatial reasoning and spatial frame axioms using choice formulas. Spatial reasoning is performed by encoding spatial relations as systems of polynomial constraints, and solving via SMT with the theory of real nonlinear arithmetic. We empirically evaluate ASPMT(QS) in comparison with other prominent contemporary spatial reasoning systems. Our results show that ASPMT(QS) is the only existing system that is capable of reasoning about indirect spatial effects (i.e. addressing the ramification problem), and integrating geometric and qualitative spatial information within a non-monotonic spatial reasoning context.

\keywords{Non-monotonic Spatial Reasoning, Answer Set Programming Modulo Theories, Declarative Spatial Reasoning}
\end{abstract}

\section{Introduction}
Non-monotonicity is characteristic of commonsense reasoning patterns concerned with, for instance, making default assumptions (e.g., about spatial inertia), counterfactual reasoning with hypotheticals (e.g., what-if scenarios), knowledge interpolation, explanation \& diagnosis (e.g., filling the gaps, causal links), belief revision. Such reasoning patterns, and therefore non-monotonicity, acquires a special significance in the context of \emph{spatio-temporal dynamics}, or computational commonsense \emph{reasoning about space, actions, and change} as applicable within areas as disparate as geospatial dynamics, computer-aided design,  cognitive vision, commonsense cognitive robotics  \cite{Bhatt:RSAC:2012}. Dynamic spatial systems are characterised by scenarios where spatial configurations of objects undergo a change as the result of interactions within a physical environment \cite{bhatt:scc:08}; this requires explicitly identifying and formalising relevant actions and events at both an ontological and (qualitative and geometric) spatial level, e.g. formalising \emph{desertification} and \emph{population displacement} based on spatial theories about \emph{appearance, disappearance, splitting, motion}, and \emph{growth} of regions \cite{bhatt2014geospatial}. This calls for a deep integration of spatial reasoning within KR-based non-monotonic reasoning frameworks \cite{bhatt2011-scc-trends}.


We select aspects of a theory of \emph{dynamic spatial systems} ---pertaining to \emph{(spatial) inertia, ramifications, causal explanation}--- that are inherent  to a broad category of dynamic spatio-temporal phenomena, and require non-monotonic reasoning \cite{bhatt:scc:08,bhatt:aaai08}. For these aspects, we provide an operational semantics and a computational framework for realising fundamental non-monotonic spatial reasoning capabilities based on Answer Set Programming Modulo Theories \cite{bartholomew2013functional}; ASPMT is extended to the qualitative spatial (QS) domain resulting in the non-monotonic spatial reasoning system ASPMT(QS). Spatial reasoning is performed in an analytic manner (e.g. as with reasoners such as CLP(QS) \cite{bhatt2011clp}), where spatial relations are encoded as systems of polynomial constraints; the task of determining whether a spatial graph $G$ is consistent is now equivalent to determining whether the system of polynomial constraints is satisfiable, i.e. Satisfiability Modulo Theories (SMT) with real nonlinear arithmetic, and can be accomplished in a sound and complete manner. Thus, ASPMT(QS) consists of a (qualitative) spatial representation module and a method for turning tight ASPMT instances into Sat Modulo Theories (SMT) instances in order to compute stable models by means of SMT solvers. 

In the following sections we present the relevant foundations of stable model semantics and ASPMT, and then extend this to ASPMT(QS) by 
defining a (qualitative) spatial representations module, and formalising spatial default reasoning and spatial frame axioms using choice formulas. We empirically evaluate ASPMT(QS) in comparison with other existing spatial reasoning systems. We conclude that ASMPT(QS) is the only system, to the best of our knowledge, that operationalises dynamic spatial reasoning within a KR-based framework.

\section{Preliminaries} \label{sec::stable_models}

\subsection{Bartholomew -- Lee Stable Models Semantics}

We adopt a definition of stable models based on syntactic transformations \cite{bartholomew2012stable}  which is a generalization of the previous definitions from \cite{ferraris2011stable} \cite{gelfond1988stable} and \cite{ferraris2005answer}. For predicate symbols (constants or variables) $u$ and $c$, expression $u \leq c$ is defined as shorthand for $\forall \textbf{x}(u(\textbf{x}) \to c(\textbf{x}))$. Expression $u = c$ is defined as $\forall \textbf{x}(u(\textbf{x}) \equiv c(\textbf{x}))$ if $u$ and $c$ are predicate symbols, and $\forall \textbf{x}(u(\textbf{x}) = c(\textbf{x}))$ if they are function symbols. For lists of symbols $\textbf{u} = (u_1, \dots , u_n )$ and $\textbf{c} = (c_1, \dots , c_n )$, expression $\textbf{u} \leq \textbf{c}$ is defined as $(u_1 \leq c_1) \land \dots \land (u_n \leq c_n )$, and similarly, expression $\textbf{u} = \textbf{c}$ is defined as $(u_1 = c_1) \land \dots \land (u_n = c_n )$. Let $\textbf{c}$ be a list of distinct predicate and function constants, and let $\widehat{\textbf{c}}$ be a list of distinct predicate and function variables corresponding to c. By $\textbf{c}^{pred}$ ($\textbf{c}^{func}$ , respectively) we mean the list of all predicate constants (function constants, respectively) in $\textbf{c}$, and by $\widehat{\textbf{c}}^{pred}$ ($\widehat{\textbf{c}}^{func}$ , respectively) the list of the corresponding predicate variables (function variables, respectively) in $\widehat{\textbf{c}}$. In what follows, we refer to function constants and predicate constants of arity $0$ as object constants and propositional constants, respectively.

\begin{definition}[\textbf{Stable model operator $\textbf{SM}$}] \label{def::stable_operator}
For any formula $F$ and any list of predicate and function constants $\textbf{c}$ (called intensional constants), $\textbf{SM}[F;\textbf{c}]$ is defined as
\begin{equation}
F \land \neg \exists \widehat{\textbf{c}} ( \widehat{\textbf{c}} < \textbf{c} \land F^*(\widehat{\textbf{c}}) ),
\end{equation}
where $\widehat{\textbf{c}} < \textbf{c}$ is a shorthand for $(\widehat{\textbf{c}}^{pred} \leq \textbf{c}^{pred}) \land \neg (\widehat{\textbf{c}} = \textbf{c})$ and $F^*(\widehat{\textbf{c}})$ is defined recursively as follows:
\begin{itemize}
\item for atomic formula $F$, $F^* \equiv F' \land F$, where $F'$ is obtained from $F$ by replacing all intensional constants $\textbf{c}$ with corresponding variables from $\widehat{\textbf{c}}$,
\item $(G \land H)^* = G^* \land H^*$, \ \ \ $(G \lor H)^* = G^* \lor H^*$,	   
\item $(G \to H)^* = (G^* \to H^*) \land (G \to H)$,
\item $(\forall x G)^* = \forall x G^*$, \ \ \ $(\exists x G)^* = \exists x G^*$.	 
\end{itemize}
$\neg F$ is a shorthand for $F \to \bot$, $\top$ for $\neg \bot$ and $F \equiv G$ for $(F \to G) \land (G \to F)$.
\end{definition}

\begin{definition}[\textbf{Stable model}] \label{def::stable_model}
For any sentence $F$, a stable model of $F$ on $\textbf{c}$ is an interpretation $I$ of underlying signature such that $I \models \textbf{SM}[F;\textbf{c}]$.
\end{definition}

\subsection{Turning ASPMT into SMT}

It is shown in \cite{bartholomew2013functional} that a tight part of ASPMT instances can be turned into SMT instances and, as a result, off-the-shelf SMT solvers (e.g. \textsc{Z3} for arithmetic over reals) may be used to compute stable models of ASP, based on the notions of Clark normal form, Clark completion.


\begin{definition}[\textbf{Clark normal form}] \label{def::clark}
Formula $F$ is in \emph{Clark normal form} (relative to the list $\textbf{c}$ of intensional constants) if it is a conjunction of sentences of the form (\ref{eqn::clark1}) and (\ref{eqn::clark2}).
\vspace{+8pt}
\begin{minipage}{0.51\linewidth}
\begin{equation} \label{eqn::clark1}
\forall \textbf{x} (G \to p(\textbf{x}))
\end{equation}
\end{minipage}
\begin{minipage}{0.48\linewidth}
\begin{equation} \label{eqn::clark2}
\forall \textbf{x}y (G \to f(\textbf{x})=y)
\end{equation}
\end{minipage}

\noindent one for each intensional predicate $p$ and each intensional function $f$, where $\textbf{x}$ is a list of distinct object variables, $y$ is an object variable, and $G$ is an arbitrary formula that has no free variables other than those in $\textbf{x}$ and $y$.
\end{definition}

\begin{definition}[\textbf{Clark completion}] \label{def::completion}
The \emph{completion} of a formula $F$ in Clark normal form (relative to $\textbf{c}$), denoted by $Comp_{\textbf{c}}[F]$ is obtained from  $F$ by replacing each conjunctive term of the form (\ref{eqn::clark1}) and (\ref{eqn::clark2}) with (\ref{eqn::completion1}) and (\ref{eqn::completion2}) respectively

\vspace{+8pt}
\begin{minipage}{0.47\linewidth}
\begin{equation} \label{eqn::completion1}
\forall \textbf{x} (G \equiv p(\textbf{x}))
\end{equation}
\end{minipage}
\begin{minipage}{0.48\linewidth}
\begin{equation} \label{eqn::completion2}
\forall \textbf{x}y (G \equiv f(\textbf{x})=y).
\end{equation}
\end{minipage}
\end{definition}

\begin{definition}[\textbf{Dependency graph}] \label{def::dependency}
The \emph{dependency graph} of a formula $F$ (relative to $\textbf{c}$) is a directed graph $DG_{\textbf{c}}[F]=(V,E)$ such that:
\begin{enumerate}
\item $V$ consists of members of $\textbf{c}$,
\item for each $c,d \in V$, $(c,d) \in E$ whenever there exists a strictly positive occurrence of $G\to H$ in $F$, such that $c$ has a strictly positive occurrence in $H$ and $d$ has a strictly positive occurrence in $G$,
\end{enumerate}
where an occurrence of a symbol or a subformula in $F$ is called strictly positive in $F$ if that occurrence is not in the antecedent of any implication in $F$.
\end{definition}

\begin{definition}[\textbf{tight formula}] \label{def::tight}
Formula $F$ is \emph{tight} (on $\textbf{c}$) if $DG_{\textbf{c}}[F]$ is acyclic.
\end{definition}

\begin{theorem}[Bartholomew, Lee] \label{the::smt}
For any sentence $F$ in Clark normal form that is tight on $\textbf{c}$, an interpretation $I$ that satisfies $\exists xy (x = y)$ is a model of $\textbf{SM}[F;\textbf{c}]$ iff $I$ is a model of $Comp_{\textbf{c}}[F]$ relative to $\textbf{c}$.
\end{theorem}

\section{ASPMT with Qualitative Space -- ASPMT(QS)} \label{sec::aspmt(qs)}

In this section we present our spatial extension of ASPMT, and formalise spatial default rules and spatial frame axioms.

\subsection{The Qualitative Spatial Domain \QS}\label{sec:qs-in-spmt}
Qualitative spatial calculi can be classified into two groups: topological and positional calculi. With topological calculi such as the \emph{Region Connection Calculus} (RCC) \cite{randell1992spatial}, the primitive entities are spatially extended regions of space, and could possibly even be 4D spatio-temporal histories, e.g., for \emph{motion-pattern} analyses. Alternatively, within a dynamic domain involving translational motion, point-based abstractions with orientation calculi could suffice. Examples of orientation calculi include: the Oriented-Point Relation Algebra ({\small\OPRAm}) \cite{moratz06_opra-ecai}, the Double-Cross Calculus \cite{freksa:1992:Orient-Qual-Spat-Reas}. The qualitative spatial domain (\QS)  that we consider in the formal framework of this paper encompasses the following ontology:

\smallskip

\noindent\mysubsecNoBook{{\color{black}QS1.\quad Domain Entities in \QS}}{XX}  Domain entities in \QS include \emph{circles, triangles, points} and \emph{segments}. While our method is applicable to a wide range of 2D and 3D spatial objects and qualitative relations, for brevity and clarity we primarily focus on a 2D spatial domain. Our method is readily applicable to other 2D and 3D spatial domains and qualitative relations, for example, as defined in \cite{pesant1994quad,bouhineau1996solving,pesant1999reasoning,bouhineau1999application,bhatt2011clp,schultz-bhatt-2012,DBLP:conf/ecai/SchultzB14}.

\begin{itemize}
	\item a \emph{point} is a pair of reals $x,y$
	\item a \emph{line segment} is a pair of end points $p_1, p_2$ ($p_1 \neq p_2$)
	\item a \emph{circle} is a centre point $p$ and a real radius $r$ ($0 < r$)
	\item a \emph{triangle} is a triple of vertices (points) $p_1, p_2, p_3$ such that $p_3$ is \emph{left of} segment $p_1, p_2$.
\end{itemize}

\noindent\mysubsecNoBook{{\color{black}QS2.\quad Spatial Relations in \QS}}{XX} 
We define a range of spatial relations with the corresponding polynomial encodings. Examples of spatial relations in \QS include:

\noindent \emph{Relative Orientation.}\quad \emph{Left, right, collinear} orientation relations between \emph{points} and \emph{segments}, and \emph{parallel, perpendicular} relations between \emph{segments} \cite{leecomplexity-ecai2014}.

\smallskip

\noindent \emph{Mereotopology.}\quad \emph{Part-whole} and \emph{contact} relations between regions \cite{varzi1996parts,randell1992spatial}.

%




%

\subsection{Spatial representations in ASPMT(QS)}

%
%
%
%

Spatial representations in ASPMT(QS) are based on parametric functions and qualitative relations, defined as follows.

\begin{definition}[\textbf{Parametric function}] \label{def::parametric}
A \emph{parametric function} is an $n$--ary function $f_n:D_1 \times D_2 \times \dots \times D_n \to \mathbb{R}$ such that for any  $i \in \{ 1 \dots n\}$, $D_i$ is a type of spatial object, e.g., $Points$, $Circles$, $Polygons$, etc. 
\end{definition}

\begin{example}
Consider following parametric functions $x:Circles \to \mathbb{R}$, $y:Circles \to \mathbb{R}$, $r:Circles \to \mathbb{R}$ which return the position values $x, y$ of a circle's centre and its radius $r$, respectively. Then, circle $c \in Cirlces$ may be described by means of parametric functions as follows: $x(c)=1.23 \land y(c)=-0.13 \land r(c)=2$. 
\end{example}

\begin{definition}[\textbf{Qualitative spatial relation}] \label{def::qualitative}
A \emph{qualitative spatial relation} is an $n$-ary predicate $Q_n \subseteq D_1 \times D_2 \times \dots \times D_n$ such that for any  $i \in \{ 1 \dots n\}$, $D_i$ is a type of spatial object. For each $Q_n$ there is a corresponding formula of the form
\begin{footnotesize}
\begin{multline} \label{eqn::qualitative}
\forall d_1 \in D_1 \dots \forall d_n \in D_n \bigg( Q_n(d_1, \dots , d_n) \leftarrow
 p_1(d_1, \dots , d_n) \land \dots \land p_m(d_1, \dots , d_n) \bigg)
\end{multline}
\end{footnotesize}
where $m \in \mathbb{N}$ and for any  $i \in \{ 1 \dots n\}$, $p_i$ is a polynomial equation or inequality. 
\end{definition}

\begin{proposition}\label{prop::qual_clark}
Each qualitative spatial relation according to Definition~\ref{def::qualitative} may be represented as a tight formula in Clark normal form.
\end{proposition}

\begin{proof}
Follows directly from Definitions~\ref{def::clark} and \ref{def::qualitative}.
\end{proof}


Thus, qualitative spatial relations belong to a part of ASPMT that may be turned into SMT instances by transforming the implications in the corresponding formulas into equivalences (Clark completion). The obtained equivalence between polynomial expressions and predicates enables us to compute relations whenever parametric information is given, and vice versa, i.e. computing possible parametric values when only the qualitative spatial relations are known.



Many relations from existing qualitative calculi may be represented in ASPMT(QS) according to Definition~\ref{def::qualitative}; our system can express the polynomial encodings presented in e.g. \cite{pesant1994quad,bouhineau1996solving,pesant1999reasoning,bouhineau1999application,bhatt2011clp}. Here we give some illustrative examples.


\begin{proposition}\label{prop::IA}
Each relation of Interval Algebra (IA) \cite{allen1983maintaining} and Rectangle Algebra (RA) \cite{guesgen1989spatial} may be defined in ASPMT(QS).
\end{proposition}

\begin{proof}
Each IA relation may be described as a set of equalities and inequalities between interval endpoints (see Figure 1 in \cite{allen1983maintaining}), which is a conjunction of polynomial expressions. RA makes use of IA relations in 2 and 3 dimensions. Hence, each relation is a conjunction of polynomial expressions \cite{DBLP:conf/ecai/SchultzB14}.
\end{proof}

%

\begin{proposition}\label{prop::RCC-5}
Each relation of RCC--5 in the domain of convex polygons with a finite maximum number of vertices may be defined in ASPMT(QS).
\end{proposition}

\begin{proof}
Each RCC--5 relation may be described by means of relations $P(a,b)$ and $O(a,b)$. In the domain of convex polygons, $P(a,b)$ is true whenever all vertices of $a$ are in the interior (inside) or on the boundary of $b$, and $O(a,b)$ is true if there exists a point $p$ that is inside both $a$ and $b$. Relations of a point being inside, outside or on the boundary of a polygon can be described by polynomial expressions e.g. \cite{bhatt2011clp}. Hence, all RCC--5 relations may be described with polynomials, given a finite upper limit on the number of vertices a convex polygon can have.
\end{proof}


\begin{proposition}\label{prop::CDC}
Each relation of Cardinal Direction Calculus (CDC) \cite{frank1991qualitative} may be defined in ASPMT(QS).
\end{proposition}

\begin{proof}
CDC relations are obtained by dividing space with 4 lines into 9 regions. Since halfplanes and their intersections may be described with polynomial expressions, then each of the 9 regions may be encoded with polynomials. A polygon object is in one or more of the 9 cardinal regions by the topological \emph{overlaps} relation between polygons, which can be encoded with polynomials (i.e. by the existence of a shared point) \cite{bhatt2011clp}.
\end{proof}

\subsection{Choice Formulas in ASPMT(QS)}


A choice formula \cite{ferraris2011stable} is defined for a predicate constant $p$ as $\Pred{Choice}(p) \equiv \forall \textbf{x} (p( \textbf{x}) \lor \neg p( \textbf{x})  ) $ and for function constant $f$ as  $\Pred{Choice}(f) \equiv \forall \textbf{x} (f( \textbf{x})=y \lor \neg f( \textbf{x})=y  ) $, where $\textbf{x}$ is a list of distinct object variables and $y$ is an object variable distinct from $\textbf{x}$. 
We use the following notation: $\{ F \}$ for $F \lor \neg F$, $\forall \textbf{x}y \{ f(\textbf{x})=y \}$ for $\Pred{Choice}(f)$ and $\forall \textbf{x} \{ p(\textbf{x}) \}$ for $\Pred{Choice}(p)$. Then, $\{ \textbf{t} = \textbf{t'} \}  $, where $\textbf{t}$ contains an intentional function constant and $\textbf{t'}$ does not, represents the default rule stating that $\textbf{t}$ has a value of $\textbf{t'}$ if there is no other rule requiring $\textbf{t}$ to take some other value.


\begin{definition}[\textbf{Spatial choice formula}] \label{def::spatial_default}
The \emph{spatial choice formula} is a rule of the form~(\ref{eqn::spatial_rule1}) or (\ref{eqn::spatial_rule2}):
\begin{equation} \label{eqn::spatial_rule2}
\{ f_n (d_1, \dots , d_n) = x \} \leftarrow \alpha_1 \land \alpha_2 \land \dots \land \alpha_k,
\end{equation}
\begin{equation} \label{eqn::spatial_rule1}
\{ Q_n (d_1, \dots , d_n) \} \leftarrow \alpha_1 \land \alpha_2 \land \dots \land \alpha_k.
\end{equation}
where $f_n$ is a parametric function, $x \in \mathbb{R}$, $Q_n$ is a qualitative spatial relation, and for each $i \in \{ 1, \dots , k \}$, $\alpha_i$ is a qualitative spatial relation or expression of a form $ \{ f_r (d_k, \dots , d_m) = y \} $ or a polynomial equation or inequality, whereas $d_i \in D_i$ is an object of spatial type $D_i$.
\end{definition}

\begin{definition}[\textbf{Spatial frame axiom}] \label{def::spatial_frame}
The \emph{spatial frame axiom} is a special case of a spatial choice formula  which states that, by default, a spatial property remains the same in the next step of a simulation. It takes the form (\ref{eqn::frame1}) or (\ref{eqn::frame2}):
\begin{equation} \label{eqn::frame1}
\{ f_n(d_1,\dots , d_{n-1} , s+1) = x \} \leftarrow f_n(d_1,\dots , d_{n-1} , s) = x,
\end{equation}
\begin{equation} \label{eqn::frame2}
\{ Q_n(d_1,\dots , d_{n-1} , s+1) \} \leftarrow Q_n(d_1,\dots , d_{n-1} , s).
\end{equation}
where $f_n$ is a parametric function, $x \in \mathbb{R}$, $Q_n$ is a qualitative spatial relation, and $s \in \mathbb{N}$ represents a step in the simulation.
\end{definition}

\begin{corollary}
One spatial frame axiom for each parametric function and qualitative spatial relation is enough to formalise the intuition that spatial properties, by default, do not change over time.
\end{corollary}

The combination of spatial reasoning with stable model semantics and arithmetic over the reals enables the operationalisation of a range of novel features within the context of dynamic spatial reasoning. We present concrete examples of such features in Section~\ref{sec::tests}.


\section{System implementation} \label{sec::implementation}

We present our implementation of ASPMT(QS) that builds on \textsc{aspmt2smt} \cite{bartholomew2014system} --  a compiler translating a tight fragment of ASPMT into SMT instances. Our system consists of an additional module for spatial reasoning and \textsc{Z3} as the SMT solver. As our system operates on a tight fragment of ASPMT, input programs need to fulfil certain requirements, described in the following section. As output, our system either produces the stable models of the input programs, or states that no such model exists.


\subsection{Syntax of Input Programs}

The input program to our system needs to be $f$-$plain$ to use Theorem 1 from \cite{bartholomew2012stable}.

\begin{definition}[\textbf{$\textbf{f-plain}$ formula}]
Let $f$ be a function constant. A first--order formula is called $f$-$plain$ if each atomic formula:
\begin{itemize}
\item does not contain $f$, or
\item is of the form $f(\textbf{t}) = u$, where $\textbf{t}$ is a tuple of terms not containing $f$, and $u$ is a term not containing $f$.
\end{itemize}
\end{definition}

Additionally, the input program needs to be \emph{av-separated}, i.e. no variable occurring in an argument of an uninterpreted function is related to the value variable of another uninterpreted function via equality \cite{bartholomew2014system}. The input program is divided into declarations of:
\begin{itemize}
\item $\Pred{sorts}$ (data types);
\item $\Pred{objects}$ (particular elements of given types);
\item $\Pred{constants}$ (functions);
\item $\Pred{variables}$ (variables associated with declared types).
\end{itemize}
The second part of the program consists of clauses. ASPMT(QS) supports:
\begin{itemize}
\item connectives: $\Pred{\&}$, $\Pred{|}$, $\Pred{not}$, $\Pred{->}$, $\Pred{<-}$, and
\item arithmetic operators: \texttt{<}, \texttt{<=}, \texttt{>=}, \texttt{>}, \texttt{=}, \texttt{!=}, \texttt{+}, \texttt{=}, \texttt{*}, with their usual meaning.
\end{itemize}   
Additionally, ASPMT(QS) supports the following as native / first-class entities:


\begin{itemize}
\item $\Pred{sorts}$ for geometric objects types, e.g., $\Pred{point}$, $\Pred{segment}$, $\Pred{circle}$, $\Pred{triangle}$;
\item parametric functions describing objects parameters e.g., $x(\Pred{point})$, $r(\Pred{circle})$;
\item qualitative relations, e.g., $\Pred{rccEC}(\Pred{circle},\Pred{circle})$, $\Pred{coincident}(\Pred{point},\Pred{circle})$.
\end{itemize}

\noindent\bul\textbf{Example 1: combining topology and size}\quad  Consider a program describing three circles $a$, $b$, $c$ such that $a$ is discrete from $b$,  $b$ is discrete from $c$, and $a$ is a proper part of $c$, declared as follows:


 \includegraphics[]{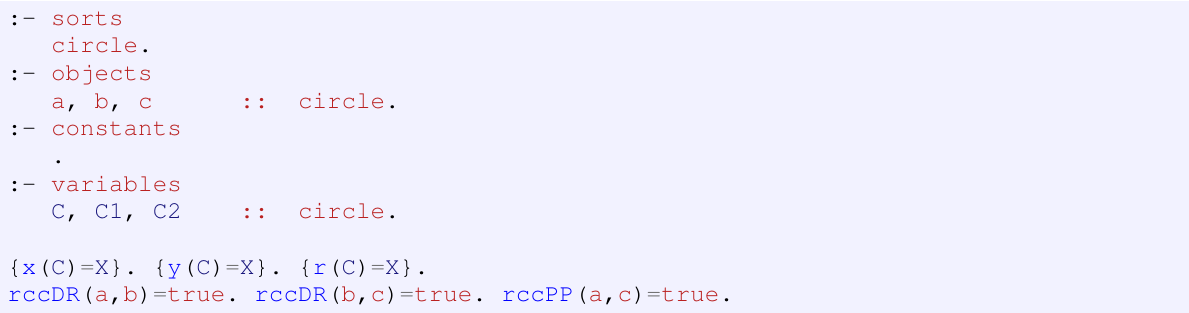}
 
%

\noindent ASPMT(QS) checks if the spatial relations are satisfiable. In the case of a positive answer, a parametric model and computation time are presented. The output of the above mentioned program is:

 \includegraphics[]{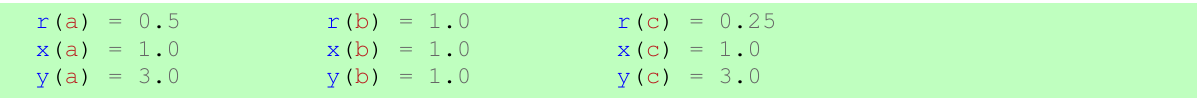}


\noindent This example demonstrates that ASPMT(QS) is capable of computing \emph{composition tables}, in this case the RCC--5 table for circles \cite{randell1992spatial}. Now, consider the addition of a further constraint to the program stating that circles $a$, $b$, $c$ have the same radius:

 \includegraphics[]{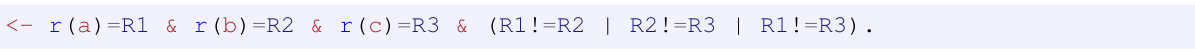}


\noindent This new program is an example of combining different \emph{types} of qualitative information, namely topology and size, which is a non-trivial research topic within the relation algebraic spatial reasoning community; relation algebraic-based solvers such as GQR \cite{gantner2008gqr,DBLP:conf/ijcai/WolflW09} will not correctly determine inconsistencies in general for arbitrary combinations of different types of relations (orientation, shape, distance, etc.). In this case, ASPMT(QS) correctly determines that the spatial constraints are inconsistent:
 
\includegraphics[]{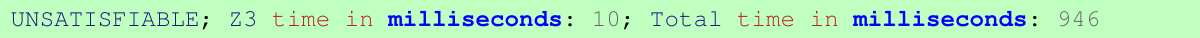}



\noindent\bul\textbf{Example 2: combining topology and relative orientation}\quad  Given three circles $a$, $b$, $c$ let $a$ be proper part of $b$,  $b$ discrete from $c$, and $a$ in contact with $c$, declared as follows: 


\begin{figure}[ht]
    \centering
    \begin{subfigure}[b]{0.5\textwidth}
        \centering
        \includegraphics[width=0.4\textwidth]{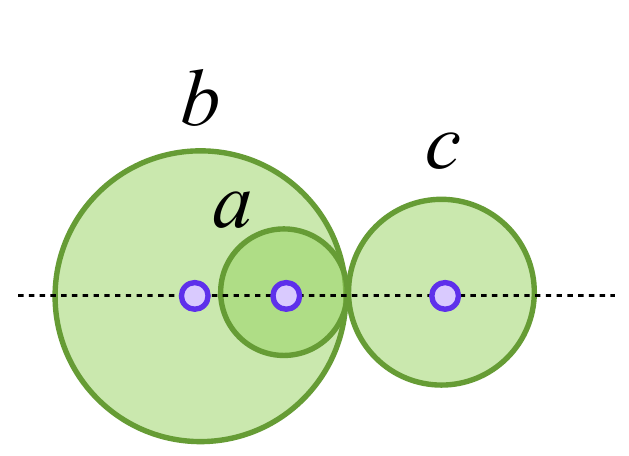}
        \caption{}
        \label{fig:topo-ori-1}
    \end{subfigure}%
    \begin{subfigure}[b]{0.5\textwidth}
        \centering
        \includegraphics[width=0.4\textwidth]{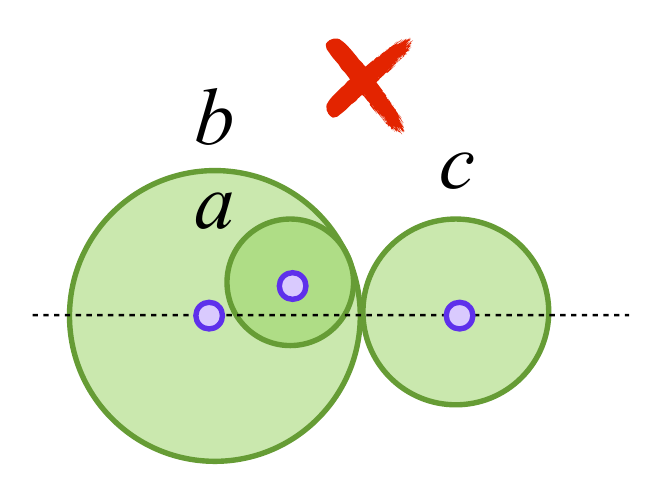}
        \caption{}
        \label{fig:topo-ori-2}
    \end{subfigure}
    \caption{Reasoning about consistent and refinement by combining topology and relative orientation.}
    \label{fig:topo-ori}
    \vspace{-10pt}
\end{figure}


%

 \includegraphics[]{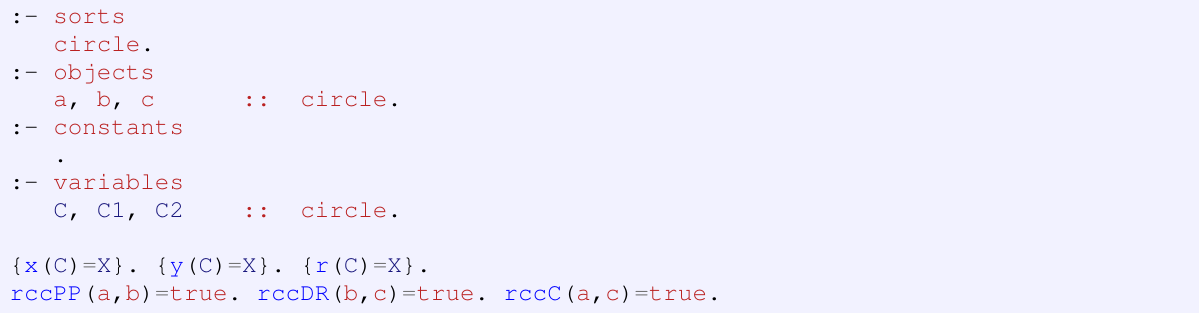}

%

\noindent Given this basic qualitative information, ASPMT(QS) is able to refine the topological relations to infer that (Figure~\ref{fig:topo-ori-1}): i) $a$ must be a \emph{tangential proper part} of $b$ ii) both $a$ and $b$ must be \emph{externally connected} to $c$.

 \includegraphics[]{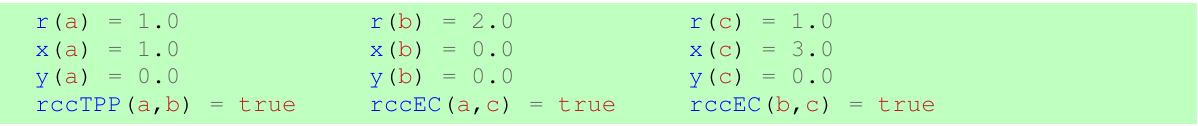}


\noindent We then add an additional constraint that the centre of $a$ is \emph{left of} the segment between the centres $b$ to $c$. 

 \includegraphics[]{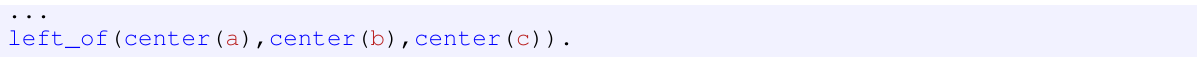}

\noindent ASPMT(QS) determines that this is inconsistent, i.e., the centres must be \emph{collinear} (Figure\;\ref{fig:topo-ori-2}).

 \includegraphics[]{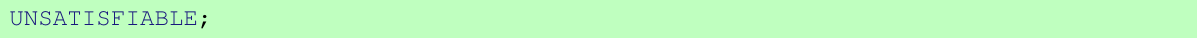}


\section{Empirical Evaluation and Examples} \label{sec::tests}

In this section we present an empirical evaluation of ASPMT(QS) in comparison with other existing spatial reasoning systems. The range of problems demonstrate the unique, non-monotonic spatial reasoning features that ASPMT(QS) provides beyond what is possible using other currently available systems.
Table~\ref{tab::cumulative} presents run times obtained by Clingo -- an ASP grounder and solver \cite{gebser2014clingo}, GQR -- a binary constraint calculi reasoner \cite{gantner2008gqr}, CLP($\mathcal{QS}$) -- a declarative spatial reasoning system \cite{bhatt2011clp} and our ASPMT(QS) implementation. Tests were performed on an Intel Core 2 Duo 2.00 GHZ CPU with 4 GB RAM running Ubuntu 14.04. The polynomial encodings of the topological relations have not been included here for space considerations.


\begin{table}[h]
  \vspace{-20pt}
\footnotesize
\begin{center}
\caption{\textit{\small Cumulative results of performed tests. ``---'' indicates that the problem can not be formalised, ``$^{I}$'' indicates that indirect effects can not be formalised, ``$^{D}$'' indicates that default rules can not be formalised.}
}
\label{tab::cumulative}
\begin{tabular}{|c||r|r|r|r|}
\hline
Problem & \multicolumn{1}{|c|}{Clingo} & \multicolumn{1}{|c|}{GQR} & \multicolumn{1}{|c|}{CLP($\mathcal{QS}$)} & \multicolumn{1}{|c|}{ASPMT(QS)}  \\
\hline
\hline
\emph{Growth} & $0.004$s$^{I}$ & $0.014$s$^{I,D}$ & $1.623$s$^{D}$ & $0.396$s  \\
\hline
\emph{Motion}  & $0.004$s$^{I}$ & $0.013$s$^{I,D}$ &  $0.449$s$^{D}$  & $15.386$s  \\
\hline
\emph{Attach I}  & $0.008$s$^{I}$ & \multicolumn{1}{|c|}{---} & $3.139$s$^{D}$  & $0.395$s  \\
\hline
\emph{Attach II} & \multicolumn{1}{|c|}{---}  & \multicolumn{1}{|c|}{---} & $2.789$s$^{D}$  & $0.642$s  \\
\hline
\end{tabular}
\end{center}
  \vspace{-20pt}
\end{table}

\subsection{Ramification Problem}
\label{sec::ramification}

\begin{wrapfigure}{r}{0.52\textwidth}
  \begin{center}
  \vspace{-47pt}
\resizebox{0.5\textwidth}{!}{
\begin{tikzpicture}

\node[draw=none,fill=none] at (-2.5,0.7)  {$S_0:$};
\node[draw=none,fill=none] at (-2.5,-1.3)  {$S_1:$};

\node (0) at (1,0.7) [line width=1.2pt,rounded corners=2pt, draw,thick,minimum width=2cm,minimum height=1.4cm] {};
\node[draw=black, circle, minimum size=11mm, inner sep=0pt,outer sep=0] at (0.7,0.7) {$ $};
\node[draw=black, fill=LightBlue, circle, minimum size=5mm, inner sep=0pt,outer sep=0] at (0.7,0.8) {$a$};
\node[draw=black, fill=LightGreen, circle, minimum size=6mm, inner sep=0pt,outer sep=0] at (1.55,0.7) {$c$};
\node[draw=none,fill=none] at (0.7,0.35)  {$b$};

\node (1) at (-0.5,-1.3) [line width=1.2pt,rounded corners=2pt, draw,thick,minimum width=2cm,minimum height=1.4cm] {};
\node[draw=black, fill=LightBlue, circle, minimum size=11mm, inner sep=0pt,outer sep=0] at (-0.8,-1.3) {$a=b$};
\node[draw=black, fill=LightGreen, circle, minimum size=6mm, inner sep=0pt,outer sep=0] at (0.05,-1.3) {$c$};

\node (or1) at (2,-1.3) [line width=1.2pt,rounded corners=2pt, draw,thick,minimum width=2cm,minimum height=1.4cm] {};
\node[draw=black, circle, minimum size=11mm, inner sep=0pt,outer sep=0] at (1.7,-1.3) {};
\node[draw=black, fill=LightBlue, circle, minimum size=5mm, inner sep=0pt,outer sep=0] at (1.7,-1.0)  {$a$};
\node[draw=black, fill=LightGreen, circle, minimum size=6mm, inner sep=0pt,outer sep=0] at (2.55,-1.3) {$c$};
\node[draw=none,fill=none] at (1.7,-1.65)  {$b$};

\node (or2) at (4.5,-1.3) [line width=1.2pt,rounded corners=2pt, draw,thick,minimum width=2cm,minimum height=1.4cm] {};
\node[draw=black, circle, minimum size=11mm, inner sep=0pt,outer sep=0] at (4.2,-1.3) {};
\node[draw=black, fill=LightBlue, circle, minimum size=5mm, inner sep=0pt,outer sep=0] at (4.5,-1.3)  {$a$};
\node[draw=black, fill=LightGreen, circle, minimum size=6mm, inner sep=0pt,outer sep=0] at (5.05,-1.3) {$c$};
\node[draw=none,fill=none] at (3.9,-1.3)  {$b$};

\node[dashed, rounded corners=2pt, fit=(or1)(or2), draw] {};

\draw[->, line width=1.2pt, rounded corners=2pt] (0,0.7) -- (-0.75,0.7) -- (-0.75,-0.6);
\draw[ line width=1.2pt, rounded corners=2pt] (2,0.7) -- (3.25,0.7) -- (3.25,0.3);

\draw[->, line width=1.2pt, rounded corners=2pt] (3.25,0.3) -- (2.25,-0.4);
\draw[->, line width=1.2pt, rounded corners=2pt] (3.25,0.3) -- (4.25,-0.4);
\node[draw=none,fill=none] at (3.25,-0.1)  {OR};

\node at (-1.75,0.2) [] {$growth(a,0)$};
\node at (4.5,0.2) [] {$motion(a,0)$};







\end{tikzpicture}
}
\caption{Indirect effects of $growth(a,0)$ and $motion(a,0)$ events.}
\label{fig::ramifications}
\end{center}
  \vspace{-29pt}
\end{wrapfigure}
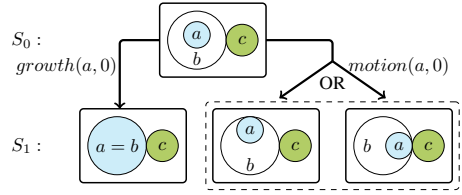

The following two problems, \emph{Growth} and \emph{Motion}, were introduced in \cite{bhatt:aaai08}. Consider the initial situation $S_0$ presented in Figure~\ref{fig::ramifications}, consisting of three cells: $a$, $b$, $c$, such that $a$ is a non-tangential proper part of $b$: $\Pred{rccNTPP}(a,b,0)$, and $b$ is externally connected to $c$: $\Pred{rccEC}(b,c,0)$.






\smallskip

$\triangleright$ \emph{Growth.} Let $a$ grow in step $S_0$; the event $\Pred{growth}(a,0)$ occurs and leads to a successor situation $S_1$. The direct effect of $growth(a,0)$ is a change of a relation between $a$ and $b$ from $\Pred{rccNTPP}(a,b,0)$ to $\Pred{rccEQ}(a,b,1)$ (i.e. $a$ is equal to $b$). No change of the relation between $a$ and $c$ is directly stated, and thus we must derive the relation $\Pred{rccEC}(a,c,1)$ as an indirect effect. 

\smallskip

$\triangleright$ \emph{Motion.} Let $a$ move in step $S_0$; the event $\Pred{motion}(a,0)$ leads to a successor situation $S_1$. The direct effect is a change of the relation $\Pred{rccNTPP}(a,b,0)$ to $\Pred{rccTPP}(a,b,1)$ ($a$ is a tangential proper part of $b$). In the successor situation $S_1$ we must determine that the relation between $a$ and $c$ can only be either $\Pred{rccDC}(a,c,1)$ or $\Pred{rccEC}(a,c,1)$.

\smallskip



GQR provides no support for domain-specific reasoning, and thus we encoded the problem as two distinct qualitative constraint networks (one for each simulation step) and solved them independently i.e. with no definition of \emph{growth} and \emph{motion}. Thus, GQR is not able to produce any additional information about indirect effects. As Clingo lacks any mechanism for analytic geometry, we implemented the RCC8 composition table and thus it inherits the incompleteness of relation algebraic reasoning. While CLP(QS) facilitates the modelling of domain rules such as \emph{growth}, there is no native support for default reasoning and thus we forced $b$ and $c$ to remain unchanged between simulation steps, otherwise all combinations of spatially consistent actions on $b$ and $c$ are produced without any preference (i.e. leading to the frame problem). 

In contrast, ASPMT(QS) can express spatial inertia, and derives indirect effects directly from spatial reasoning: in the \emph{Growth} problem ASPMT(QS) abduces that $a$ has to be concentric with $b$ in $S_0$ (otherwise a \emph{move} event would also need to occur). Checking global consistency of scenarios that contain interdependent spatial relations is a crucial feature that is enabled by a support polynomial encodings and is provided only by CLP(QS) and ASPMT(QS).

\subsection{Geometric Reasoning and the Frame Problem}

In problems \emph{Attachment I} and \emph{Attachment II} the initial situation $S_0$ consists of three objects (circles), namely $\Const{car}$, $\Const{trailer}$ and $\Const{garage}$ as presented in Figure~\ref{fig::car}. Initially, the $\Const{trailer}$ is attached to the $\Const{car}$: $\Pred{rccEC}(\Const{car},\Const{trailer},0)$,  $\Pred{attached}(\Const{car},\Const{trailer},0)$. The successor situation $S_1$ is described by $\Pred{rccTPP}(\Const{car},\Const{garage},1)$. The task is to infer the possible relations between the trailer and the garage, and the necessary actions that would need to occur in each scenario.

There are two domain-specific actions: the car can move, $\Pred{move}(\Const{car},X)$, and the trailer can be detached, $\Pred{detach}(\Const{car},\Const{trailer},X)$ in simulation step $X$. Whenever the $\Const{trailer}$ is attached to the $\Const{car}$, they remain $\Pred{rccEC}$. The $\Const{car}$ and the $\Const{trailer}$ may be either completely outside or completely inside the $\Const{garage}$. 



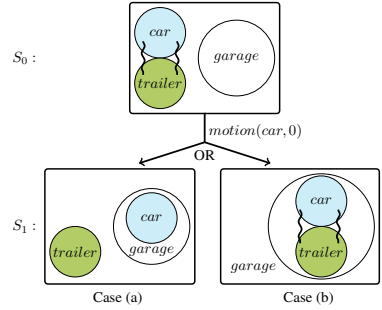
\begin{wrapfigure}{r}{0.40\textwidth}
  \begin{center}
  \vspace{-33pt}
\resizebox{0.41\textwidth}{!}{
\begin{tikzpicture}

\node[draw=none,fill=none] at (-2.7,1)  {$S_0:$};

\node (garage0) [draw=black, circle, minimum size=18mm, inner sep=0pt,outer sep=0] at (2.3,1) {$garage$};
\node (car0) [draw=black, fill=LightBlue, circle, minimum size=12mm, inner sep=0pt,outer sep=0] at (0.5,1.6) {$car$};
\node (trailer0) [draw=black, fill=LightGreen, circle, minimum size=12mm, inner sep=0pt,outer sep=0] at (0.5,0.4) {$trailer$};
\node[line width=1.0pt, rounded corners=2pt, fit=(garage0)(car0)(trailer0), draw] {};

\draw[-,very thick,decorate,decoration={snake,amplitude=1}] (0.1,1.4) -- (0.1,0.6);
\draw[-,very thick,decorate,decoration={snake,amplitude=1}] (0.9,1.4) -- (0.9,0.6);

\node[draw=none,fill=none] at (-2.7,-3)  {$S_1:$};

\node (garage3) [draw=white, circle, minimum size=25mm, inner sep=0pt,outer sep=0] at (-0.3,-3) {$ $};

\node (garage1) [draw=black, circle, minimum size=18mm, inner sep=0pt,outer sep=0] at (0.3,-3) {$ $};
\node (car1) [draw=black, fill=LightBlue, circle, minimum size=12mm, inner sep=0pt,outer sep=0] at (0.3,-2.8) {$car$};
\node (trailer1) [draw=black, fill=LightGreen, circle, minimum size=12mm, inner sep=0pt,outer sep=0] at (-1.5,-3.6) {$trailer$};
\node[line width=1.0pt, rounded corners=2pt, fit=(garage1)(car1)(trailer1)(garage3), draw] {};
\node[draw=none,fill=none] at (0.3,-3.6)  {$garage$};


\node (garage1) [draw=black, circle, minimum size=25mm, inner sep=0pt,outer sep=0] at (4.3,-3) {$ $};
\node (car1) [draw=black, fill=LightBlue, circle, minimum size=12mm, inner sep=0pt,outer sep=0] at (4.3,-2.4) {$car$};
\node (trailer1) [draw=black, fill=LightGreen,  circle, minimum size=12mm, inner sep=0pt,outer sep=0] at (4.3,-3.6) {$trailer$};
\node (garage3) [draw=none,fill=none] at (2.7,-4.0) {$garage$};
\node[line width=1.0pt, rounded corners=2pt, fit=(garage1)(car1)(garage3)(trailer1), draw] {};

\draw[-,very thick,decorate,decoration={snake,amplitude=1}] (3.8,-2.6) -- (3.8,-3.4);
\draw[-,very thick,decorate,decoration={snake,amplitude=1}] (4.7,-2.6) -- (4.7,-3.4);

\draw[ line width=1.2pt, rounded corners=2pt] (1.55,-0.3) -- (1.55,-1.0);

\draw[->, line width=1.2pt, rounded corners=2pt] (1.55,-1.0) -- (0.0,-1.5);
\draw[->, line width=1.2pt, rounded corners=2pt] (1.55,-1.0) -- (3.1,-1.5);
\node[draw=none,fill=none] at (1.55,-1.3)  {OR};

\node at (2.75,-0.75) [] {$motion(car,0)$};

\node at (-0.5,-4.7) [] {Case (a)};
\node at (4,-4.7) [] {Case (b)};

\end{tikzpicture}
}
\caption{Non-monotonic reasoning with additional geometric information.}
\label{fig::car}
\end{center}
  \vspace{-30pt}
\end{wrapfigure}

$\triangleright$ \emph{Attachment I.} Given the available topological information, we must infer that there are two possible solutions (Figure.~\ref{fig::car}); (a) the $\Const{car}$ was detached from the $\Const{trailer}$ and then moved into the $\Const{garage}$:
(b) the $\Const{car}$, together with the $trailer$ attached to it, moved into the $\Const{garage}$:

$\triangleright$ \emph{Attachment II.} We are given additional geometric information about the objects' size: $r(car)=2$, $r(trailer)=2$ and $r(garage)=3$. Case (b) is now inconsistent, and we must determine that the only possible solution is (a).

These domain-specific rules require default reasoning: ``\emph{typically} the $\Const{trailer}$ remains in the same position'' and ``\emph{typically} the $\Const{trailer}$ remains attached to the $\Const{car}$''. The later default rule is formalised in ASPMT(QS) by means of the spatial defaul.:
The formalisation of such rules addresses the frame problem. GQR is not capable of expressing the domain-specific rules for detachment and attachment in \emph{Attachment I} and \emph{Attachment II}. Neither GQR nor Clingo are capable of reasoning with a combination of topological and numerical information, as required in \emph{Attachment II}. As CLP(QS) cannot express default rules, we can not capture the notion that, for example, the trailer should typically remain in the same position unless we have some explicit reason for determining that it moved; once again this leads to an exhaustive enumeration of all possible scenarios without being able to specify preferences, i.e. the frame problem, and thus CLP(QS) will not scale in larger scenarios.

The results of the empirical evaluation show that ASPMT(QS) is the only system that is capable of (a) non-monotonic spatial reasoning, (b) expressing domain-specific rules that also have spatial aspects, and (c) integrating both qualitative and numerical information. Regarding the greater execution times in comparison to CLP(QS), we have not yet implemented any optimisations with respect to spatial reasoning; this is one of the directions of future work.

\section{Conclusions} \label{sec::conclusions}

We have presented ASPMT(QS),  a novel approach for reasoning about spatial change within a KR paradigm. By integrating dynamic spatial reasoning within a KR framework, namely answer set programming (modulo theories), our system can be used to model behaviour patterns that characterise high-level processes, events, and activities as identifiable with respect to a general characterisation of commonsense \emph{reasoning about space, actions, and change} \cite{Bhatt:RSAC:2012,bhatt:scc:08}. ASPMT(QS) is capable of sound and complete spatial reasoning, and combining qualitative and quantitative spatial information when reasoning non-monotonically; this is due to the approach of encoding spatial relations as polynomial constraints, and solving using SMT solvers with the theory of real nonlinear arithmetic. We have demonstrated that no other existing spatial reasoning system is capable of supporting the key non-monotonic spatial reasoning features (e.g., spatial inertia, ramification) provided by ASPMT(QS) in the context of a mainstream knowledge representation and reasoning method, namely, answer set programming.

\subsubsection*{Acknowledgments.} This research is partially supported by: (a) the Polish National Science Centre grant 2011/02/A/HS1/0039; and (b). the DesignSpace Research Group \texttt{www.design-space.org}.

\bibliographystyle{splncs03}
\bibliography{ijcai15}

\end{document}